\documentclass{article}
\usepackage{ijcai15}

\usepackage{times, tikz}
\usepackage{todonotes}
\usepackage{amsmath}
\usepackage{amsfonts}
\usepackage{stmaryrd}
\usepackage{microtype}
\usepackage{booktabs}
\usetikzlibrary{shapes}
\newtheorem{theorem}{Theorem}

\newtheorem{definition}{Definition}

\newtheorem{proposition}{Proposition} 
\newtheorem{example}{Example} 

\newenvironment{proof}{\noindent \textit{Proof:}}{\hfill $\square$}

\newcommand{\EL}{\ensuremath{\mathcal{EL}}}

\newcommand{\ELext}{\ensuremath{\mathcal{ELU}}}
\newcommand{\ALC}{\ensuremath{\mathcal{ALC}}}

\newcommand{\K}{\ensuremath{\mathcal{K}}}

\newcommand{\NC}{\ensuremath{\mathcal{N}_C}}
\newcommand{\NR}{\ensuremath{\mathcal{N}_R}}

\newcommand{\C}[1]{\ensuremath{\mathsf{C}({#1})}}
\newcommand{\Cn}[1]{\ensuremath{\mathsf{Cn}({#1})}}

\newcommand{\nat}{\ensuremath{\mathbb{N}}}

\newcommand{\ModAll}[1]{\ensuremath{\mathrm{Mod}\left(#1\right)}}

\newcommand{\Sen}[1]{\ensuremath{\mathrm{Sen}\left(#1\right)}}

\newcommand{\I}{\ensuremath{\mathcal{I}}}

\newcommand{\Cmc}{\ensuremath{\mathcal{C}}}

\newcommand{\Smc}{\ensuremath{\mathcal{S}}}

\usepackage{ifthen}
\newcommand{\isdraft}{\boolean{true}} 

\newcommand{\markupdraft}[2]{
    \ifthenelse{\equal{#1}{display}}{#2}{}
    \ifthenelse{\equal{#1}{color}}{\color{#2}}{}
}

\newcommand{\newcolored}[3][]{{\markupdraft{color}{#2}#3}
    \ifthenelse{\equal{#1}{}}{}{\markupdraft{display}{{\color{yellow!70!black}[#1]}}}} 

\newcommand{\new}[1]{#1}

\ifthenelse{\isdraft}{}{\renewcommand{\markupdraft}[2]{}}

\def\presuper#1#2%
  {\mathop{}%
   \mathopen{\vphantom{#2}}^{#1}%
   \kern-\scriptspace%
   #2}
   
\title{Relaxation-based revision operators in description logics}

\author{Marc Aiguier$^1$, Jamal Atif$^2$, Isabelle Bloch$^3$ and C\'eline Hudelot$^1$ \\
1. Centrale Sup\'elec, MAS, France, {\it \{marc.aiguier,celine.hudelot\}@centralesupelec.fr} \\
2. PSL, Universit\'e Paris-Dauphine, LAMSADE, UMR 7243, France, {\it jamal.atif@dauphine.fr} \\
3. Institut Mines-Telecom, Telecom ParisTech, CNRS LTCI, France, {\it isabelle.bloch@telecom-paristech.fr}}
\begin{document}

\maketitle

\begin{abstract}
 {\small As ontologies and description logics (DLs) reach out to a broader audience, several reasoning services are developed in this context. 
Belief revision is one of them, of prime importance when knowledge is prone to change and inconsistency. In this paper we  address both the generalization of the well-known AGM postulates, and the definition of concrete and well-founded revision operators in different DL families. 
We introduce a model-theoretic version of the AGM postulates with a general definition of inconsistency, hence enlarging their scope to a wide family of non-classical logics, in particular negation-free DL families. 
We propose a general framework for defining revision operators based on the notion of relaxation, introduced recently for defining dissimilarity measures between DL concepts. A revision operator in this framework amounts to relax the set of models of the old belief until it reaches the sets of models of the new piece of knowledge. 
We demonstrate that such a relaxation-based revision operator defines a faithful assignment and satisfies the generalized AGM postulates. 
Another important contribution concerns the definition of several concrete relaxation operators suited to the syntax of some DLs ($\ALC{}$ and its fragments $\EL{}$ and $\ELext{}$). }

  \end{abstract}
  
\section{Introduction}
\label{sec:introduction}
Belief revision is at the core of artificial intelligence and philosophy questionings. It is defined as the process of changing an agent belief with a new acquired  knowledge. Three change operations are usually considered: {\em  expansion, contraction} and {\em revision}. 
We focus here on the revision, i.e. the process of adding {\bf consistently} the new belief sets. 
Belief revision has been intensively studied in classical logics (e.g. propositional logic) mostly under the prism of AGM theory~\cite{AGM-85}. With the growing interest in non-classical logics, such as Horn Logics and Description Logics~\cite{Baad03nocrossref}, several attempts to generalize AGM theory, making it compliant to the meta-logical flavor of these logics, have been introduced recently~\cite{flouris2005applying,delgrande2015,ribeiro2013,ribeiro2014}.  

In this paper, we are interested in defining concrete revision operators in Description Logics (DLs). DLs  are now pervasive in many knowledge-based representation systems, such as ontological reasoning, semantic web, scene understanding, cognitive robotics, to mention a few. In all these domains, the expert knowledge is rather a flux evolving through time, requiring hence the definition of rational revision operators. Revision is then a cornerstone in ontology engineering life-cycle where the expert knowledge is prone to change and inconsistency. This paper contributes to the effort of defining such  rational revision operators, in line with the recent art of the domain~\cite{qi2006,qi2006revision,flouris2005applying,flouris2006inconsistencies}. 
In Section~\ref{sec:Preliminaries} 
we discuss the adaptation of AGM theory to non-classical logics, including DLs, and introduce, as a first contribution,  a model-theoretic rewriting of AGM postulates. In Section~\ref{sec:theoryrelaxation}, we introduce our general framework of relaxation-based revision operators. As a second contribution, we demonstrate that they satisfy the AGM postulates and lead to a faithful assignment. Our third contribution is detailed in Section~\ref{sec:concrete}, by providing concrete theory relaxation operators in different DLs (namely $\ALC{}$ and its fragments $\EL{}$ and $\ELext{}$). Section~\ref{sec:related} positions our contributions with regards to the literature and finally Section~\ref{sec:conc} draws some conclusions and perspectives.


\section{Preliminaries}
\label{sec:Preliminaries}
\subsection{Description Logics}
\label{sec:DL}
Description logics are a family
of knowledge representation formalisms (see e.g.~\cite{Baad03nocrossref} for more details). We consider, in this paper, the logic $\ALC{}$ and its fragments $\EL{}$ and $\ELext{}$. In the following we provide the syntax and semantics of $\ALC{}$, from which those of $\EL{}$ and $\ELext{}$ are easily deducible.  Signatures in DLs are triplets $(N_C,N_R,I)$ where $N_C$, $N_R$ and $I$ are nonempty pairwise disjoint sets such that elements in $N_C$, $N_R$ and $I$ are {\em concept names}, {\em role names} and {\em individuals}, respectively. Given a signature $\Sigma \in \mathrm{Sign}$, $\Sen{\Sigma}$ contains all the sentences of the form $C \sqsubseteq D$, $x:C$ and $(x,y):r$ where $x,y \in I$, $r \in N_R$ and $C$ is an $\ALC{}$-concept inductively defined from $N_C$ and binary and unary operators in $\{\_ \sqcap \_,\_ \sqcup \_\}$ and in $\{\_^c,\forall r.\_,\exists r.\_\}$, respectively. The set of  \emph{concept descriptions} provided by $\Sigma$ is denoted by $\C{\Sigma}$. The semantics of concept descriptions is
defined using interpretations. An \emph{interpretation} $\I$ is a pair $\I =
(\Delta_\I, \cdot^\I)$  consisting of an interpretation domain $\Delta_\I$ and
an interpretation function $\cdot^\I$ which maps concepts to subsets of the
domain $\Delta_\I$ and role names to binary relations on the domain.
A concept description $C$ is said to \emph{subsume} a concept description $D$
(denoted by $C \sqsubseteq D$) if $C^\I \subseteq D^\I$ holds for every
interpretation $\I$. Two concepts $C$ and $D$ are \emph{equivalent} (denoted by
$C \equiv D$) if both $C \sqsubseteq D$ and $D \sqsubseteq C$ hold. An interpretation $\I$ is a model of a $\Sigma$-sentence (TBox or ABox axiom) if it satisfies this sentence (e.g. $\I \models_\Sigma (C \sqsubseteq D)\text{~iff~}C^\I \subseteq D^\I$).

A DL knowledge base $T$ is a set of $\Sigma$-sentences (i.e. $T \subseteq \Sen{\Sigma}$). An interpretation $\I$ is a model of a DL knowledge base $T$ if it satisfies every sentence in $T$. 
In the following, we use $\ModAll\varphi$ (or $\ModAll{T}$) to denote the set of all the models of a $\Sigma$-sentence $\varphi$ (or DL knowledge base $T$). 
 A knowledge base is said to be a {\em theory} if and only if $T = \Cn{T}$, where $\Cn{}$ is the consequence operator defined as: $\Cn{T}=\{\varphi \in \Sen{\Sigma} \mid \forall \I \in \ModAll{T}, \I \models_\Sigma \varphi\}$ and satisfying {\em monotonicity}, {\em inclusion} and {\em idempotence}.  Hence DLs can be considered as Tarskian logics,  i.e. pairs $\langle \Sigma, \Cn{} \rangle$.

Classically, consistency of a theory $T$ in DLs is defined as $\ModAll{T}\neq \emptyset$. Such a definition raised several issues in adapting revision  postulates to Description Logics (see~\cite{ribeiro2013,ribeiro2014}). We consider in this paper a more general definition of consistency the meaning of which is that there is at least a sentence which is not a semantic consequence: $T \subseteq \Sen{\Sigma}$ is {\em consistent} if $\Cn{T} \neq \Sen{\Sigma}$.

\subsection{AGM theory and Description Logics}
AGM theory~\cite{AGM-85} is probably the most influential paradigm in belief revision theory~\cite{gardenfors2003belief}. It provides, at an abstract logical level, a set of postulates  that a revision operator should satisfy so that the old belief  is changed minimally and rationally to become consistent with the new one. These postulates require the logic to be closed under negation and usual propositional connectives $(\vee, \wedge, \implies, \neg)$ which prevents its use in many non-classical logics, including DLs. Indeed, many DLs do not allow for negation of concepts (e.g. $\EL{}$) and {\em a fortiori} disjunction between TBox and ABox sentences is not defined in all DLs.   Recently, many papers have addressed the adaptation of AGM theory to non-classical logics, e.g.~\cite{flouris2005applying,ribeiro2013,delgrande2015,ribeiro2014}. The first efforts concentrated on the adaptation of contraction postulates, but more recently,~\cite{ribeiro2014} discussed the adaptation of revision postulates and introduced new minimality criteria, not necessarily related to the contraction  operator, throwing out the need for negation. However, one can find in~\cite{qi2006} an attempt to adapt the AGM revision postulates to DL in a model-theoretic way, following the seminal work of~\cite{KatsunoMendelzon91} that translated the AGM postulates in propositional logic semantics. The translation in~\cite{qi2006} is provided with the classical notion of consistency (a theory $T$ is consistent if $\ModAll{T}\neq 0$) which is not adequate to revision purposes (see~\cite{flouris2006inconsistencies} for a discussion). In this paper we  consider a model-theoretic translation of AGM postulates, similar to the ones in~\cite{qi2006}, with the notable difference that consistency is defined through the consequence operator as introduced in the previous section. This translation is in accordance with the postulates as introduced in~\cite{ribeiro2014}. 

Given two knowledge bases $T, T' \subseteq \Sen{\Sigma}$, $T \circ T'$ denotes the revision of the old belief $T$ by the new one $T'$. The model-theoretic translation of  AGM postulates writes:
\begin{description}
\item[(G1)] $\ModAll{T \circ T'} \subseteq \ModAll{T'}$.
\item[(G2)] If $T \cup T'$ is consistent, then $T \circ T' = T \cup T'$.
\item[(G3)] If $T'$ is consistent, then so is $T \circ T'$.
\item[(G4)] If $\ModAll{T_1} = \ModAll{T'_1}$ and $\ModAll{T_2} = \ModAll{T'_2}$, then $\ModAll{T_1 \circ T_2} = \ModAll{T'_1 \circ T'_2}$.
\item[(G5)]  $\ModAll{T \circ T'} \cap \ModAll{T''} \subseteq \ModAll{T \circ(T'\cup T'')}$.  
\item[(G6)] If $T \cup T' \cup T''$ is consistent, then $T \circ (T' \cup T'') = (T \circ T') \cup T''$. 
\end{description}

Besides these postulates, we consider a minimality criterion introduced in~\cite{ribeiro2014}:\\
{\bf (Relevance)} If $\varphi \in T \setminus (T \circ T')$, then there exists $X$, $T \cap (T \circ T') \subseteq X \subseteq T$, such that $\Cn{X \cup T'} \neq \Sen{\Sigma}$ and $\Cn{X \cup \{\varphi\} \cup T'} = \Sen{\Sigma}$.

A classical construction in belief theory is to characterize the revision operator in terms of {\em faithful assignments}~\cite{KatsunoMendelzon91,grove1988two}. We provide here a similar representation theorem for the postulates defined above. The proof can be found in~\cite{ARXIV}.
%
\begin{definition}[Faithful assignment]
Let $T \subseteq \Sen{\Sigma}$ be a knowledge base. Let $\preceq_T \subseteq \ModAll{\Sigma} \times \ModAll{\Sigma}$ be a total pre-order. $\preceq_T$ is a {\em faithful assignment (FA)} if the following three conditions are satisfied:

\begin{enumerate}
\item If $\I,\I' \in \ModAll{T}$, $\I {\not \prec}_T \I'$.
\item For every $\I \in \ModAll{T}$ and every $\I' \in \ModAll{\Sigma} \setminus \ModAll{T}$, $\I \prec_T \I'$.
\item For every $T' \subseteq \Sen{\Sigma}$, if $\ModAll{T} = \ModAll{T'}$, then $\preceq_T = \preceq_{T'}$.
\end{enumerate} 
\end{definition}

\begin{theorem}
A revision operator $\circ$ satisfies the postulates {\bf (G1)-(G6)} if and only if for any DL knowledge base $T$, there exists a well-founded (i.e. the min is well defined) FA $\preceq_{T}$ such that $\ModAll{T \circ T'} = \min(\ModAll{T'} \setminus M^*,\preceq_{T})$, with $M^* = \{ \I \in \ModAll\Sigma \mid \{ \varphi \in \Sen\Sigma \mid \I \models_\Sigma \varphi \} = \Sen\Sigma \}$.
\end{theorem}

\section{Relaxation of theories and  associated revision operator}
\label{sec:theoryrelaxation}
The notion of relaxation has been introduced in~\cite{FD:ECAI-14,FD:KR-14} to define dissimilarity measures between DL concept descriptions. In this paper we generalize this notion to formula relaxation and subsequently to theory relaxation in order to define revision operators. 

\begin{definition}[Concept Relaxation~\cite{FD:KR-14}]
\label{def:ConceptRelaxation}
  A \emph{(concept) relaxation} is an operator $\rho \colon \C{\Sigma}
  \rightarrow \C{\Sigma}$ that satisfies the following two properties\footnote{The non-decreasingness property in the original definition is removed here, since it is not needed in our construction.} for all
  $C \in  \C{\Sigma}$.
  \begin{enumerate}
    \item $\rho$ is \emph{extensive}, i.e.\ $C \sqsubseteq \rho(C)$, 
    \item $\rho$ is \emph{exhaustive}, i.e.\ $\exists k \in \nat_0 \colon
      \top \sqsubseteq \rho^k(C)$,
  \end{enumerate}
  where $\rho^k$ denotes $\rho$ applied $k$ times, and $\rho^0$ is the
  identity mapping.
\end{definition}


Our idea to define revision operators is to relax the set of models of the old belief until it becomes consistent with the new pieces of knowledge. This is illustrated in Figure~\ref{fig:relaxation} where theories are represented as sets of their models. Intermediate steps to define the revision operators are then the definition of  formula  and theory relaxations.  The whole scheme of our framework is provided in Figure~\ref{fig:structure}. 

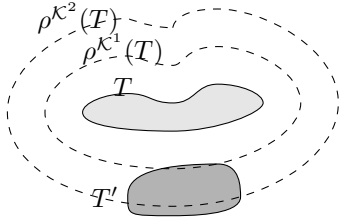
\begin{figure}[t]
  \newlength{\ddx}
  \newlength{\ddy}
  \begin{center}
    \begin{tikzpicture}
  \setlength{\ddx}{0.6cm}
  \setlength{\ddy}{.125cm}
  \newlength{\shift}
  \setlength{\shift}{.5cm}
  \newlength{\cshift}
  \setlength{\cshift}{.15cm}
  \path[draw, fill, fill opacity = 0.3] (\ddx, -6.5\ddy) ..
    controls (1.5\ddx, -6.5\ddy) and (1.5\ddx, -9\ddy)
    .. (1.5\ddx, -10\ddy) ..
    controls (1.5\ddx, -12\ddy) and (1.5\ddy, -12\ddy)
    .. (0, -12\ddy) ..
    controls (-\ddx, -12\ddy) and (-\ddx, -12\ddy)
    .. (-\ddx, -10\ddy) ..
    controls (-\ddx, -9\ddy) and (-\ddx, -6.5\ddy)
    .. (\ddx, -6.5\ddy);
  \node at (-1.5\ddx, -10\ddy) {$T'$};
  \path[draw, fill, fill opacity = 0.1] (0, 0) ..
    controls (.5\ddx, 0) and (.5\ddx, 2\ddy)
    .. (\ddx, 2\ddy) ..
    controls (1.5\ddx, 2\ddy) and (2\ddx, \ddy)
    .. (2\ddx, 0) ..
    controls (2\ddx, -\ddy) and (\ddx, -3\ddy)
    .. (0, -3\ddy) ..
    controls (-\ddx, -3\ddy) and (-2\ddx, -2.5\ddy)
    .. (-2\ddx, -\ddy) ..
    controls (-2\ddx, 0) and (-1.5\ddx, \ddy)
    .. (-\ddx, \ddy) ..
    controls (-.5\ddx, \ddy) and (-.5\ddx, 0)
    .. (0, 0);
  \node at (-1.1\ddx, 1.7\ddy) {$T$};
  \path[draw, dashed] (0, \shift) ..
    controls (.5\ddx - \cshift, \shift) and (.5\ddx - \cshift, 2\ddy + \shift)
    .. (\ddx, 2\ddy + \shift) ..
    controls (1.5\ddx + \cshift, 2\ddy + \shift) and
      (2\ddx + \shift, \ddy + 2\cshift)
    .. (2\ddx + \shift, 0) ..
    controls (2\ddx + \shift, -\ddy - 2\cshift) and
      (\ddx + 2\cshift, -3\ddy - \shift)
    .. (0, -3\ddy - \shift) ..
    controls (-\ddx - 2\cshift, -3\ddy - \shift) and
      (-2\ddx - \shift, -2.5\ddy - 2\cshift)
    .. (-2\ddx - \shift, -\ddy) ..
    controls (-2\ddx - \shift, 2\cshift) and
      (-1.5\ddx - \cshift, \ddy + \shift)
    .. (-\ddx, \ddy + \shift) ..
    controls (-.5\ddx - \cshift, \ddy + \shift) and (-.5\ddx, \shift)
    .. (0, \shift);
  \node at (-1.1\ddx, 1.7\ddy + \shift) {$\rho^{\K^1}(T)$};
  \setlength{\cshift}{.3cm}
  \setlength{\shift}{1cm}
  \path[draw, dashed] (0, \shift) ..
    controls (.5\ddx - \cshift, \shift) and (.5\ddx - \cshift, 2\ddy + \shift)
    .. (\ddx, 2\ddy + \shift) ..
    controls (1.5\ddx + \cshift, 2\ddy + \shift) and
      (2\ddx + \shift, \ddy + 2\cshift)
    .. (2\ddx + \shift, 0) ..
    controls (2\ddx + \shift, -\ddy - 2\cshift) and
      (\ddx + 2\cshift, -3\ddy - \shift)
    .. (0, -3\ddy - \shift) ..
    controls (-\ddx - 2\cshift, -3\ddy - \shift) and
      (-2\ddx - \shift, -2.5\ddy - 2\cshift)
    .. (-2\ddx - \shift, -\ddy) ..
    controls (-2\ddx - \shift, 2\cshift) and
      (-1.5\ddx - \cshift, \ddy + \shift)
    .. (-\ddx, \ddy + \shift) ..
    controls (-.5\ddx - \cshift, \ddy + \shift) and (-.5\ddx, \shift)
    .. (0, \shift);
  \node at (-2.1\ddx, 1\ddy + \shift) {$\rho^{\K^2}(T)$};
\end{tikzpicture}
  \end{center}
  \caption{Relaxations of $T$ until it becomes consistent with $T'$. \label{fig:relaxation}}
\end{figure}

\begin{definition}[Formula Relaxation]
Given a signature $\Sigma \in \mathrm{Sign}$, a {\em $\Sigma$-formula relaxation} is a mapping  $\rho_{\Sigma} \colon \Sen{\Sigma} \to \Sen{\Sigma}$ satisfying the following properties:
\begin{itemize}
\item Extensivity:  $\forall \varphi \in  \Sen{\Sigma}, \ModAll{\varphi} \subseteq \ModAll{\rho_{\Sigma}(\varphi)}$.
\item Exhaustivity: $\exists k\in \mathbb{N}, \ModAll{\rho_{\Sigma}^k(\varphi)}=\ModAll{\Sigma}$, where $\rho^k$denotes $\rho_{\Sigma}$ applied $k$ times, and $\rho_{\Sigma}^0$ is the identity mapping.
\end{itemize}
\end{definition}

\begin{definition}[$\Sigma$-theory relaxation]
\label{def:TR}
Let $T$ be a  theory, $T \in \mathcal{P}(\Sen{\Sigma}) $, $\rho_{\Sigma}$ a $\Sigma$-formula relaxation and a set $\K=\{k_\varphi \in \mathbb{N} \mid \varphi \in T \}$. Then a {\em  $\Sigma$-theory relaxation }  is  a mapping $\rho^{\mathcal{K}} \colon \mathcal{P}(\Sen{\Sigma})\rightarrow \mathcal{P}(\Sen{\Sigma})$ defined  as:
\[
\rho^\K(T)=\bigcup_{\varphi \in T} \rho^{k_\varphi}_\Sigma(\varphi).
\]
\end{definition}

\begin{proposition}
$\rho^{\mathcal{K}}$
is extensive ($\forall T \subseteq  \Sen{\Sigma}, \ModAll{T}\subseteq \ModAll{\rho^{\mathcal{K}}(T)}$),
and exhaustive ($\exists \K \subseteq \mathbb{N}, \ModAll{\rho^\K(T)}=\ModAll{\Sigma}$).
\end{proposition}

\begin{definition}[Relaxation-based  revision]
\label{revisionrelaxation}
Let $\rho^{\K}$ be a $\Sigma$-theory relaxation. We define the revision operator $\circ : \mathcal{P}(Sen(\Sigma)) \times \mathcal{P}(Sen(\Sigma)) \to \mathcal{P}(Sen(\Sigma))$ as follows: 
$$T_1 \circ T_2 = \rho^\K(T'_1)  \cup T''_1 \cup T_2$$
for a set $\K$ such that   $\rho^\K(T'_1) \cup T''_1 \cup T_2$ is consistent, and    $\forall \K' \text{ s.t. } \rho^{\K'}(T'_1) \cup T''_1 \cup T_2 \text{ is consistent}, \sum_{k \in \K} k \leq \sum_{k \in \K'} k$,  and $T_1 = T'_1 \coprod T''_1$ (disjoint union) such that:
\begin{enumerate}
\item $Cn(T'_1 \cup T_2) = Sen(\Sigma)$,
\item $Cn(T''_1 \cup T_2) \neq Sen(\Sigma)$,
\item $\forall T \text{~s.t.~} T''_1 \subset T \subseteq T_1, Cn(T \cup T_2) = Sen(\Sigma)$.
\end{enumerate}  
\end{definition} 

Partitioning $T_1$ into $T'_1$ and $T''_1$ is not unique
and the only constraint is that $T''_1$ is of maximal size.
The set $\K$ may not be unique either.

\begin{theorem}
\label{orders and relaxation}
From any  $\Sigma$-theory relaxation $\rho^{\mathcal{K}}$ and  every knowledge base $T \subseteq \Sen{\Sigma}$, the binary relation $\preceq_T \subseteq \ModAll{\Sigma} \times \ModAll{\Sigma}$ defined by $\mathcal{I} \preceq_T \mathcal{I'}$ if:
\[
\min_{\mathcal{K} \mid \mathcal{I} \in \ModAll{\rho^\mathcal{K}(T)}} \sum_{k \in \mathcal{K}} k \leq \min_{\mathcal{K'} \mid \mathcal{I'} \in \ModAll{\rho^\mathcal{K'}(T)}} \sum_{k \in \mathcal{K'}} k
\]
is a well-founded faithful assignment such that for every $T' \subseteq \Sen{\Sigma}$, 
$\ModAll{T \circ T'} = \min(\ModAll{T'} \setminus M^*,\preceq_T)$. 
\end{theorem}

\begin{proof}
By construction, $\preceq_T$ is obviously a total pre-order. Well-foundness follows from exhaustivity. The two first conditions follow from the fact that $\I\in\ModAll{T} \iff  \min_{\K \mid  \mathcal{I} \in \ModAll{\rho^\mathcal{K}(T)}}\sum_{k \in \K} k =0 $. The third one is obvious.

It remains to show that 
$\ModAll{T \circ T'} = \min(\ModAll{T'} \setminus M^*,\preceq_T)$. To simplify the proof, we suppose that $T \circ T' = \rho^\K(T) \cup T'$ (i.e. if $T = T_1 \coprod T_2$, then  $T_2 = \emptyset$), the more general case where $T_2 \neq \emptyset$ being easily obtained from this more simple case.

(i) Let $\I \in \ModAll{T \circ T'}$. By definition of $\circ$, there exists a set $\K \subset \mathbb{N}$ such that $\I \in \ModAll{\rho^{\K}(T) \cup T'}$, and then $\I  \in \ModAll{T'}$. Let $\I' \in \ModAll{T'}$. If $\I' \in \ModAll{\rho^\K(T)}$, then $\I {\not \preceq}_T  \I'$ and $\I' {\not \preceq}_T  \I$. Otherwise $\I' \not\in \ModAll{\rho^\K(T)}$ and $\forall \K'$ such that $\I' \in \ModAll{\rho^{\K'}(T)}$ we have $\sum_{k\in K'}  k\geq \sum_{k\in K} k$. Then $\min_{\mathcal{K'} \mid \mathcal{I'} \in \ModAll{\rho^\mathcal{K'}(T)}}  \sum_{k \in \K'} \geq \sum_{k \in \K} k$, which implies  $\I \preceq_T \I'$. We can hence conclude that 
$\I \in \min(\ModAll{T'} \setminus M^*,\preceq_T)$.

(ii) Conversely, let 
$\I \in \min(\ModAll{T'} \setminus M^*,\preceq_T)$. By definition of $\circ$, there exists 
a set $\K$ of minimal sum such that $\rho^\K(T) \cup T'$ is consistent and $T \circ T' = \rho^\K(T) \cup T'$. As 
$\I \in  \min(\ModAll{T'} \setminus M^*,\preceq_T)$, this means that, for every $\I' \in \ModAll{\rho^\K(T)\cup T'}$,  $\I \preceq_T \I'$, and then $\I \in \ModAll{\rho^\K(T)\cup T'} = \I \in \ModAll{T \circ T'}$.  
\end{proof}

\begin{proposition}
The revision in Definition~\ref{revisionrelaxation} satisfies the {\bf Relevance} minimality criterion.
\end{proposition}
The proof is direct by setting $X = T''_1$.

\section{Concrete theory relaxations in different $\ALC{}$ fragments}
\label{sec:concrete}

In this section, we introduce concrete relaxation operators suited to the syntax of  the logic $\ALC{}$, as defined in Section~\ref{sec:DL}, and its fragments $\EL{}$ and $\ELext{}$. $\EL{}$-concept description constructors are existential restriction ($\exists$), conjunction ($\sqcap$), $\top$ and $\bot$, while $\ELext{}$-concept constructors are those of $\EL{}$ enriched with disjunction ($\sqcup$).

Formulas in DL are of the GCI form: $C\sqsubseteq D$, where $C$ and $D$ are any two complex concepts, or Abox assertions: ($a:C, \langle a,b\rangle : r)$, with $r$ a role.  We propose to define a $\Sigma$-formula relaxation in two ways (other definitions may also exist). For GCIs, a first approach consists in relaxing the set of models of $D$ while another one amounts to  ``retract'' the set of models of $C$.

\begin{definition}
\label{def:formula_dilation}
Let $C$ and $D$ be any two complex concepts defined over the signature $\Sigma$. The \emph{concept relaxation based $\Sigma$-formula relaxation}, denoted $\presuper{r}{\rho_{\Sigma}}$ is defined as follows:
\begin{multline*}
\presuper{r}{\rho_{\Sigma}}(C\sqsubseteq D) \equiv C \sqsubseteq \rho(D) \\
\presuper{r}{\rho_{\Sigma}}(a: C) \equiv a : \rho(C), 
\; \; \; 
\presuper{r}{\rho_{\Sigma}}(\langle a,b\rangle : r)) \equiv \langle a,b\rangle : r_\top
\end{multline*}   
where $r^\I_\top=\Delta^\I \times \Delta^\I$ and $\rho$ is a concept relaxation as in Definition~\ref{def:ConceptRelaxation}.
\end{definition}

\begin{proposition}
$\presuper{r}{\rho_{\Sigma}}$ is a $\Sigma$-formula relaxation, that is extensive and exhaustive.
\end{proposition}
The proof directly follows from the extensivity and exhaustivity of $\rho$.

\begin{definition}[Concept Retraction]
\label{concepterosion}
A \emph{(concept) retraction} is an operator $\kappa \colon \C{\Sigma}
  \rightarrow \C{\Sigma}$ that satisfies the following three properties for all
  $C, D \in \C{\Sigma}$.
  \begin{enumerate}
    \item $\kappa$ is \emph{anti-extensive}, i.e.\ $\kappa(C) \sqsubseteq C $, and 
    \item $\kappa$ is \emph{exhaustive}, i.e. $\forall D\in \C{\Sigma}, \exists k \in \mathbb{N} \mid \kappa^k(C) \sqsubseteq D$,
  \end{enumerate}
  where $\kappa^k$ denotes $\kappa$ applied $k$ times, and $\kappa^0$ is the
  identity mapping.
\end{definition} 


\begin{definition}
\label{def:formula_erosion}
Let $C$ and $D$ be any two complex concepts defined over the signature $\Sigma$. The \emph{concept retraction based $\Sigma$-formula relaxation}, denoted $\presuper{c}{\rho_{\Sigma}}$ is defined as follows:
\[
\presuper{c}{\rho_{\Sigma}}(C\sqsubseteq D) \equiv \kappa(C) \sqsubseteq D \\
\]   
where $\kappa$ is a concept retraction.
\end{definition}
The definition for Abox assertions is similar as in Definition~
\ref{def:formula_dilation}.

\begin{proposition}
$\presuper{c}{\rho_{\Sigma}}$ is a $\Sigma$-formula relaxation..
\end{proposition}

For coming up with revision operators, it remains to define concrete relaxation and retraction operators at the concept level (cf. Figure~\ref{fig:structure}). Some examples of retraction and relaxation operators are given below.

\subsection{Relaxation and retraction in $\mathcal{EL}$}

\paragraph{$\EL{}$-Concept Retractions.}
A trivial concept retraction is the operator $\kappa_{\bot}$ that maps every concept to $\bot$. This operator is particularly interesting for debugging ontologies expressed in $\cal{EL}$~\cite{schlobach2007debugging}. Let us illustrate this operator through the following example adapted from~\cite{qi2006} to restrict the language to $\cal{EL}$.

\begin{example}
\label{ex:tweety}
{\small
Let $T=\{\textsc{Tweety}\sqsubseteq \textsc{bird}, \textsc{bird}\sqsubseteq \textsc{flies}\}$ and $T'=\{\textsc{Tweety} \sqcap \textsc{flies} \sqsubseteq \bot\}$. Clearly $T \cup T'$ is inconsistent. The retraction-based   $\Sigma$-formula relaxation amounts to apply $\kappa_\bot$ to the concept $\textsc{Tweety}$ resulting in the following new knowledge base $\{\bot \sqsubseteq \textsc{bird}, \textsc{bird}\sqsubseteq \textsc{flies}\}$ which is now consistent with $T'$. An alternative solution is to retract the concept $\textsc{bird}$ in $\textsc{bird}\sqsubseteq \textsc{flies}$ which results in the following knowledge base $\{\textsc{Tweety} \sqsubseteq \textsc{bird}, \bot \sqsubseteq \textsc{flies}\}$ which is also consistent with $T'$. The  sets of minimal sum $\mathcal{K}_1$ and $\mathcal{K}_2$  in Definition~\ref{revisionrelaxation} are $\mathcal{K}_1=\{1,0\}$, (i.e. $k_{\varphi_1}=1, k_{\varphi_2}=0$, where $\varphi_1=\textsc{Tweety}\sqsubseteq \textsc{bird}, \varphi_2= \textsc{bird}\sqsubseteq \textsc{flies}$) and  $\mathcal{K}_2=\{0,1\}$. A good  final solution  could be $T\circ T'=\{\bot \sqsubseteq \textsc{bird}, \textsc{bird}\sqsubseteq \textsc{flies},\textsc{Tweety} \sqcap \textsc{flies} \sqsubseteq \bot \}$ based on an additional preference relation among the solutions defined from the minimality of the ``size'' of the modified concepts.
}
\end{example}

\paragraph{$\EL{}$-Concept Relaxations.}
A trivial relaxation is the operator $\rho_\top$ that maps every concept to $\top$. Other non-trivial $\EL{}$-concept description relaxations have been introduced in~\cite{FD:ECAI-14}. We summarize here some of these operators.

\EL{} concept descriptions can appropriately be
represented as labeled trees, often called \emph{\EL{} description trees}
\cite{baader1999computing}. An \EL{} description tree is a tree whose nodes are
labeled with sets of concept names and whose edges are labeled with role names.
An \EL{} concept description
\begin{equation}
  \label{eqn:normalForm}
  C \equiv P_1 \sqcap \cdots \sqcap P_n
  \sqcap \exists r_1. C_1 \sqcap \cdots \sqcap \exists r_m. C_m
\end{equation}
with $P_i \in \NC \cup \{\top\}$, can be translated into a description tree by
labeling the root node $v_0$ with $\{P_1, \dots, P_n\}$, creating an $r_j$
successor, and then proceeding inductively by expanding $C_j$ for the
$r_j$-successor node for all $j \in \{1, \dots, m\}$. 

An $\EL{}$-concept description relaxation then amounts  to apply simple tree  operations. 
Two relaxations can hence be defined~\cite{FD:ECAI-14}: (i) $\rho_\text{depth}$ that
reduces the role depth of each concept by $1$, simply by pruning the
description tree, and (ii) $\rho_\text{leaves}$ that removes all leaves from a description
tree.

 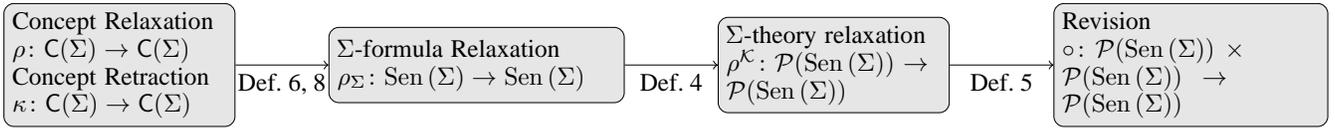
\begin{figure*}[t]
  \begin{center}
    \tikzstyle{block} = [rectangle, draw, fill=black!10,
                         rounded corners, minimum height=3em]
    \resizebox{\linewidth}{!}{
      \begin{tikzpicture}[node distance = 5cm, auto]
        \node[block, text width=3cm] (metric)
          {Concept Relaxation \\ $\rho \colon \C{\Sigma}  \to \C{\Sigma}$\\
          Concept Retraction \\
          $\kappa \colon \C{\Sigma}  \to \C{\Sigma}$};
        \node[block, text width=3.9cm, right of = metric] (dilation)
          {$\Sigma$-formula Relaxation \\ $\rho_\Sigma \colon  \Sen{\Sigma} \to \Sen{\Sigma}$};
        \node[block, text width=3cm, right of = dilation] (relaxation)
          {$\Sigma$-theory relaxation \\ $\rho^{\cal K} \colon \mathcal{P}(\Sen{\Sigma})   \to \mathcal{P}(\Sen{\Sigma})$};
       \node[block, text width=3.6cm, right of = relaxation] (dissimilarity)
          {Revision \\ $\circ \colon \mathcal{P}(\Sen{\Sigma}) \times \mathcal{P}(\Sen{\Sigma})  \to \mathcal{P}(\Sen{\Sigma})$};
        \path[->]
          (metric) edge node[below] { Def.~\ref{def:formula_dilation}, \ref{def:formula_erosion} } (dilation)
          (dilation) edge node[below] {Def.~\ref{def:TR}} (relaxation)
          (relaxation) edge node[below] {Def.~\ref{revisionrelaxation}}
            (dissimilarity);
      \end{tikzpicture}
    }
  \end{center}
  \caption{From concept relaxation and retraction to  revision operators in DL.\label{fig:structure}}
\end{figure*}

\subsection{Relaxations in $\mathcal{ELU}$}
The relaxation defined above exploits the strong property that an $\EL{}$ concept description is isomorphic to a description tree. This is arguably not true for more expressive DLs. Let us try to go a one step further in expressivity and consider the logic $\ELext$. 
A relaxation operator as introduced in~\cite{FD:ECAI-14} requires a concept description to be in a special normal form, called normal form with grouping of existentials, defined recursively as follows.
\begin{definition}
  We say that an $\EL$-concept $D$ is written in \emph{normal
  form with grouping of existential restrictions} if it is of the form
  \begin{equation}
  \label{eqn:normalFormWithGrouping}
    D = \bigsqcap_{A \in N_D} A \sqcap \bigsqcap_{r \in \NR} D_{r},
  \end{equation}
  where $N_D \subseteq \NC$ is a set of concept names and the concepts $D_{r}$
  are of the form
  \begin{equation}
  \label{eqn:groupedExistentials}
    D_{r} = \bigsqcap_{E \in \Cmc_{D_r}} \exists r.E,
  \end{equation}
  where no subsumption relation holds between two distinct conjuncts and
  $\Cmc_{D_r}$ is a set of complex $\EL$-concepts that are themselves in
normal form with grouping of existential restrictions. 
\end{definition}

The purpose of
  $D_r$ \new{terms} is simply to group existential restrictions that share the
  same role name.
  For an \ELext-concept $C$ we say that $C$ is in \emph{normal form} if it is
  of the form ($C\equiv C_1\sqcup C_2\sqcup \cdots \sqcup C_k$) and each of the $C_i$ is an
  \EL-concept in normal form with grouping of existential restrictions.
\begin{definition}\cite{FD:ECAI-14}
Given an $\ELext$-concept description $C$ we define an operator $\rho_e$
recursively as follows.
For $C = A \in \NC$ and for $C = \top$ we define
  $\rho_e(A) = \rho_e(\top) = \top$.
For $C = D_r$, where $D_r$ is a group of existential restrictions as in
(\ref{eqn:groupedExistentials}), we need to distinguish two cases:
\begin{itemize}
  \item if $D_r \equiv \exists r.\top$ we define $\rho_e(D_r) = \top$, and
  \item if $D_r \not \equiv \exists r.\top$ then we define
$      \rho_e(D_r) = \bigsqcup_{\Smc \subseteq \Cmc_{D_r}}
        \left(
          \bigsqcap_{E \notin \Smc}
            \exists r. E \sqcap
            \exists r.\rho_e \bigg(\bigsqcap_{F \in \Smc} F\bigg)
        \right)$.
\end{itemize}
Notice that in the latter case $\top \notin \Cmc_{D_r}$ since $D_r$ is
in normal form. For $C = D$ as in (\ref{eqn:normalFormWithGrouping}) we define
$  \rho_e(D) = \bigsqcup_{G \in \Cmc_D} \bigg(\rho_e(G) \sqcap
    \bigsqcap_{H \in \Cmc_D \setminus G} H\bigg)$,
where $\Cmc_D = N_D \cup \{D_r \mid r \in \NR\}$. Finally for $C = C_1 \sqcup
C_2 \sqcup \cdots \sqcup C_k$ we set
$  \rho_e(C) = \rho_e(C_1) \sqcup \rho_e(C_2)
    \sqcup \cdots \sqcup \rho_e(C_k)$.
\label{def:relxaDilate}
\end{definition}
 
 The proof of $\rho_e$ being a relaxation, i.e. satisfying exhaustivity and extensivity is detailed in~\cite{FD:TechRep-14}.
 
 Let us illustrate this operator on an example. 
 \begin{example}
{\small 
Suppose  an agent believes that a person $\textsc{Bob}$  is married to a  female judge: $T=\{\textsc{Bob}  \sqsubseteq  \textsc{male} \sqcap \exists.\textsc{MarriedTo}.\left(\textsc{female}\sqcap \textsc{judge}\right) \}$. Suppose now that due to some obscurantist law, it happens that  females are not allowed to be judges. This new belief is captured as $T'=\{\textsc{judge}\sqcap \textsc{female} \sqsubseteq \bot\}$. By applying $\rho_e$ one can resolve the conflict between the two belief sets. To ease the reading, let us rewrite the concepts as follows: $A\equiv \textsc{male}, B\equiv \textsc{female}, C\equiv \textsc{judge}, m\equiv\textsc{MarriedTo}, D \equiv \exists\textsc{MarriedTo}.\left(\textsc{female}\sqcap \textsc{judge}\right)$. Hence $\rho_e(A\sqcap D) \equiv \left(\rho_e(A) \sqcap D\right) \sqcup \left(A \sqcap \rho_e(D)\right) $, with $\rho_e(A)\equiv \top$ and 
 \[
 \begin{split}
 \rho_e(D)\equiv & \exists m.\rho_e(B\sqcap C) \sqcup \left(\exists m.B \sqcap \exists m.\rho_e(C)\right) \sqcup \\& \left( \exists m.\rho_e(B)\sqcap \exists m.C\right)\\
 \equiv&\exists m.(B \sqcup C) \sqcup \left(\exists m.B \sqcap \exists m.\top \right) \sqcup \left( \exists m.\top \sqcap \exists m.C\right)\\
 \equiv   &\exists m.B \sqcup \exists m.C \sqcup \exists m.(B \sqcup C)
 \equiv \exists m.B \sqcup \exists m.C
 \end{split}
 \]
 Then 
  \[
 \begin{split}
\rho_e(A\sqcap D)\equiv & \left(\rho_e(A) \sqcap D\right) \sqcup \left(A \sqcap \rho_e(D)\right) \\
 \equiv&(\top \sqcap D) \sqcup \left(A \sqcap \left(\exists m.B \sqcup \exists m.C\right)\right)\\
 \equiv &  D  \sqcup \left(A \sqcap \left(\exists m.B \sqcup \exists m.C\right)\right)
 \end{split}
 \]
The new agent's belief, up to a rewriting, becomes\\ $\{\textsc{Bob} \sqsubseteq  \exists.\textsc{MarriedTo}.\left(\textsc{female}\sqcap \textsc{judge}\right) \sqcup \left(\textsc{male} \sqcap  \left(\exists \textsc{Married}.\textsc{female} \sqcup \exists \textsc{Married}.\textsc{judge}\right)\right), \\\textsc{judge}\sqcap \textsc{female} \sqsubseteq \bot\}.$
}
 \end{example}
 
 Another possibility for defining a relaxation in $\ELext{}$ is  obtained by exploiting the disjunction constructor by augmenting a concept description with a set of exceptions.
 
 \begin{definition}
 \label{def:rel_ex}
 Given an exception set $\mathcal{E}=\{E_1,\cdots,E_n\}$, we define a relaxation of degree $k$ of an $\ELext{}$-concept description $C$ as follows:
for a finite set $\mathcal{E}^k \subseteq \mathcal{E}$ with $|\mathcal{E}^k|=k$ 
 \[
 \rho^k_{\mathcal{E}}(C)=C\sqcup E_{i_1} \sqcup \cdots \sqcup E_{i_k}, \forall i_j,  E_{i_j} \in \mathcal{E}^k  \text{ and } E_{i_j} \sqcap C \sqsubseteq \bot
 \]
 \end{definition}
Extensivity  of this operator follows directly from the definition. However, exhaustivity is not necessarily satisfied unless the exception set includes the $\top$ concept or the disjunction of some or all of its elements entails the $\top$ concept. 

{\small\it If we consider again Example~\ref{ex:tweety}, a relaxation of the formula $\textsc{bird}\sqsubseteq \textsc{flies}$ using the operator $ \rho^k_{\mathcal{E}}$ over the concept $\textsc{flies}$ with the exception set $\mathcal{E}=\{\textsc{Tweety}\}$ results in the formula $\textsc{bird}\sqsubseteq \textsc{flies}\sqcup \textsc{Tweety}$. The new revised knowledge base is then $\{\textsc{Tweety}\sqsubseteq \textsc{bird}, \textsc{bird}\sqsubseteq \textsc{flies}\sqcup \textsc{Tweety},\textsc{Tweety} \sqcap \textsc{flies} \sqsubseteq \bot\}$ which is consistent.}
 
 Another example involving this relaxation will be discussed in the $\ALC{}$ case (cf. Example~\ref{ex:rich}).
 
 \subsection{Relaxation and retraction in $\ALC{}$}
 We consider here operators suited to $\ALC{}$ language. Of course, all the operators defined for $\EL{}$ and $\ELext{}$ remain valid.
 
 \paragraph{$\ALC{}$-Concept Retractions.}
 A first possibility for defining retraction is to remove iteratively from an $\ALC{}$-concept description one or a set of its subconcepts. A similar construction has been introduced in~\cite{qi2006} by transforming Abox assertions to nominals and conjuncting their negations from the concept they belong to. Interestingly enough, almost all the operators defined in~\cite{qi2006,Gorogiannis2008a} are relaxations. 
 \begin{definition}
 Let $C$ be any $\ALC{}$-concept description, we define 
$ \kappa^n_{\cal E}(C)=C\sqcap  E^c_1 \sqcap \cdots\sqcap  E^c_n \text{ s.t. }E_1\sqsubseteq C, \cdots, E_n \sqsubseteq C$.
 \end{definition}
  
{\small\it Consider again Example~\ref{ex:tweety}. $ \kappa^1_{\cal E}( \textsc{bird})= \textsc{bird}\sqcap \textsc{Tweety}^c$. The resulting revised knowledge base is then $\{\textsc{Tweety}\sqsubseteq \textsc{bird}, \textsc{bird} \sqcap  \textsc{Tweety}^c\sqsubseteq \textsc{flies},\textsc{Tweety} \sqcap \textsc{flies} \sqsubseteq \bot\}$ which is consistent.}

 Another possibility, suggested in~\cite{Gorogiannis2008a} and related to operators defined in propositional logic as introduced in~\cite{TCIS-02}, consists in applying the retraction at the atomic level. This captures somehow the Dalal's idea of revision operators in propositional logic~\cite{DALA-88}.

 \begin{definition}
 \label{def:dalal-er}
 Let $C$ be an $\ALC{}$-concept description of the form $Q_1r_1\cdots Q_mr_m.D$, where $Q_i$ is a quantifier and $D$ is  quantifier-free and in CNF form, i.e. $D=E_1\sqcap E_2\sqcap\cdots E_n$ with $E_{i}$ being disjunctions of possibly negated atomic concepts. Define, as in the propositional case~\cite{TCIS-02}, $\kappa_p(D)=\bigsqcap_{j=1}^n(\bigsqcup_{i\neq j} E_i)$. Then 
$ \kappa^n_{\text{Dalal}}(C)=Q_1r_1\cdots Q_nr_n.\kappa^n_p(D)$.
 \end{definition}

 
 This idea can be generalized to consider any retraction defined in $\ELext{}$.
 
  \begin{definition}
 Let $C$ be an $\ALC{}$-concept description of the form $Q_1r_1\cdots Q_mr_m.D$, where $Q_i$ is a quantifier and $D$ is a quantifier-free. Then 
$ \kappa^n_{\cap}(C)=Q_1r_1\cdots Q_mr_m.\kappa^n_{\cal E}(D)$.
 \end{definition}
 
 Another possible $\ALC{}$-concept description retraction is obtained by substituting  the existential restriction by an universal one. This idea has been sketched in~\cite{Gorogiannis2008a} for defining dilation operators (then by transforming $\forall$ into $\exists$), i.e. special relaxation operators enjoying additional properties~\cite{FD:ECAI-14,FD:ECAI-14}.  We adapt it here to define retraction in DL syntax.
   \begin{definition}
 Let $C$ be an $\ALC{}$-concept description of the form $Q_1r_1\cdots Q_nr_n.D$, where $Q_i$ is a quantifier and $D$ is quantifier-free, then 
 \begin{align*}
 \kappa_{q}(C)=\bigsqcap \{Q'_1r_1\cdots Q'_nr_n.D \mid 
    \exists j\leq n \text{ s.t. } Q_j=\exists \\ \text { and } Q'_j=\forall,  \text{ and for all }  i\leq n  
 \text{ s.t. } i\neq j, Q'_i=Q_i \}
 \end{align*}
 \label{def:retract_q}
 \end{definition}

  \paragraph{$\ALC{}$-Concept Relaxations.}

 Let us now introduce some relaxation operators suited to $\ALC{}$ language.
  \begin{definition}
 Let $C$ be an $\ALC{}$-concept description of the form $Q_1r_1\cdots Q_mr_m.D$, where $Q_i$ is a quantifier and $D$ is quantifier-free and in DNF form, i.e. $D=E_1\sqcup E_2\sqcup\cdots E_n$ with $E_{i}$ being conjunction of possibly negated atomic concepts. Define, as in the propositional case~\cite{TCIS-02}, $\rho_p(D)=\bigsqcup_{j=1}^n(\bigsqcap_{i\neq j} E_i)$, then 
$ \rho^n_{\text{Dalal}}(C)=Q_1r_1\cdots Q_mr_m.\rho_p^n(D)$.
 \end{definition}
 
 As for retraction, this idea can be generalized to consider any relaxation defined in $\ELext{}$.
 
  \begin{definition}
 Let $C$ be an $\ALC{}$-concept description of the form $Q_1r_1\cdots Q_nr_n.D$, where $Q_i$ is a quantifier and $D$ is  quantifier-free, then 
$ \rho^n_{\cup}(C)=Q_1r_1\cdots Q_nr_n.\rho^n_{\cal E}(D).$
 \end{definition}
 
  Let us consider another example adapted from the literature to illustrate these operators~\cite{qi2006}.
 \begin{example}
 \label{ex:rich}
{\small\it Let us consider the following knowledge bases: $T=\{\textsc{Bob}\sqsubseteq \forall \textsc{hasChild}.\textsc{rich}, \textsc{Bob}\sqsubseteq \exists  \textsc{hasChild}.\textsc{Mary}, \textsc{Mary}\sqsubseteq \textsc{rich} \}$ and $T'=\{\textsc{Bob}\sqsubseteq \textsc{hasChild}.\textsc{John}, \textsc{John}\sqsubseteq  \textsc{rich}^c\}$. Relaxing  the formula $\textsc{Bob}\sqsubseteq \forall \textsc{hasChild}.\textsc{rich}$ by applying $\rho^n_{\cup}$ to the concept on the right hand side results in the following formula  $\textsc{Bob}\sqsubseteq \forall \textsc{hasChild}.(\textsc{rich}\sqcup \textsc{John})$ which resolves the conflict between the two knowledge bases. }
 \end{example}
 
 A last possibility, dual to the retraction operator given in Definition~\ref{def:retract_q}, consists in transforming universal quantifiers to existential ones.
 
    \begin{definition}
 Let $C$ be an $\ALC{}$-concept description of the form $Q_1r_1\cdots Q_nr_n.D$, where $Q_i$ is a quantifier and $D$ is  quantifier-free, then 
 \begin{align*}
 \rho_{q}(C)=\bigsqcup \{Q'_1r_1\cdots Q'_nr_n.D \mid 
   \exists j\leq n \text{ s.t. } Q_j=\forall \\ \text{ and } Q'_j=\exists, \text{ and for all }  i\leq n  
 \text{ s.t. } i\neq j, Q'_i=Q_i \}
 \end{align*}
 \label{def:relax_q}
 \end{definition}
{\small\it If we consider again Example~\ref{ex:rich}, relaxing the formula $\textsc{Bob}\sqsubseteq \forall \textsc{hasChild}.\textsc{rich}$ by applying $\rho_{q}$ to the concept on the right hand side results in the following formula  $\textsc{Bob}\sqsubseteq \exists \textsc{hasChild}.\textsc{rich}$, which resolves the conflict between the two knowledge bases. }

The following proposition summarizes the properties of the introduced operators.
\begin{proposition}
The operators $\rho_\top,  \rho_\text{depth}, \rho_\text{leaves}, \rho_e,\rho_{\text{Dalal}}, \rho_q$ are extensive and exhaustive. 
The operators $ \rho_{\mathcal{E}}, \rho_{\cup}$ are extensive but not exhaustive.
The operators $\kappa_\bot,  \kappa_{\mathcal{E}},\kappa_{\text{Dalal}}, \kappa_{\cap}$ are anti-extensive and exhaustive. 
The operators $\kappa_q$ is anti-extensive but not exhaustive.
\end{proposition}
These properties are directly derived from the definitions and from properties of $\rho_p$ and $\kappa_p$ detailed in~\cite{TCIS-02}. Note that for $\kappa_{q}$ exhaustivity can be obtained by further removing recursively the remaining universal quantifiers and apply at the final step any retraction defined above on the concept $D$.

\section{Related works}
\label{sec:related}
In the last decade, several works have studied revision operators in Description Logics. While most of them concentrated on the adaptation of AGM theory, few works have concerned the definition of concrete operators~\cite{meyer2005knowledge,qi2006revision,qi2006}. A closely related field is inconsistency handling in ontologies (e.g.~\cite{schlobach2003non,schlobach2007debugging}), with the main difference that the rationality of inconsistency  repairing operators is not investigated, as suggested by AGM theory.

Some of our relaxation operators are closely related to the ones introduced in~\cite{qi2006} for knowledge bases revision and in~\cite{Gorogiannis2008a} for merging first-order theories. Our relaxation-based revision framework, being abstract enough (i.e. defined through easily satisfied properties), encompasses these operators.  Moreover,  the revision operator defined in~\cite{qi2006} considers only  inconsistencies due to Abox assertions. Our operators are general in the sense that Abox assertions are handled as any formula of the language. 

 The relaxation idea originates from the work on Morpho-Logics, initially introduced in~\cite{TCIS-02,IB:KR-04}. In this seminal work, revision operators (and explanatory relations) were defined through dilation and erosion operators. These operators share some similarities with relaxation and retraction as defined in this paper. Dilation is a sup-preserving operator and erosion is inf-preserving, hence both are increasing. Some particular dilations and erosions are exhaustive and extensive while relaxation and retraction operators are defined to be exhaustive and extensive but not necessarily sup- and inf-preserving. 
 
 Another contribution in this paper concerns the generalization of AGM postulates and their translation in a model-theoretic writing with a  definition of inconsistency, allowing using them in a wide class of non-classical logics. This follows recent  works on the adaptation of AGM theory (e.g.~\cite{ribeiro2013,ribeiro2014,delgrande2015,flouris2005applying}). Our generalization is closely related to the one recently introduced in~\cite{ribeiro2014} and could be seen as its counterpart in a model-theoretic setting. It also extends the one introduced in~\cite{qi2006}.

\section{Conclusion}
\label{sec:conc}

The contribution of this paper is threefold. First, we
provided a generalization of AGM postulates so as they become applicable to a wide class of non-classical logics. Secondly we proposed a general framework for defining revision operators based on the notion of relaxation. We demonstrated that such a relaxation-based framework for belief revision satisfies the AGM postulates and leads to a faithful assignment. Thirdly, we introduced a bunch of concrete relaxations, discussed their properties and illustrated them through simple examples. Future work will concern the study of  the complexity of these operators,  the comparison of their induced ordering, and their generalization  to other non-classical logics such as Horn logic. 
\paragraph{Acknowledgments.} This work was partially funded by the French ANR project LOGIMA.



\bibliographystyle{named}
\bibliography{DL_revision_MM}

\begin{thebibliography}{}

\bibitem[\protect\citeauthoryear{Aiguier \bgroup \em et al.\egroup
  }{2015}]{ARXIV}
M.~Aiguier, J.~{Atif}, I.~{Bloch}, and C.~Hudelot.
\newblock Belief revision in institutions: A relaxation-based approach.
\newblock {\em CoRR}, abs/1502.02298, 2015.

\bibitem[\protect\citeauthoryear{Alchourr\'on \bgroup \em et al.\egroup
  }{1985}]{AGM-85}
C.~E. Alchourr\'on, P.~G\"ardenfors, and D.~Makinson.
\newblock On the {L}ogic of {T}heory {C}hange: {P}artial {M}eet {C}ontraction
  and {R}evision {F}unctions.
\newblock {\em Journal of Symbolic Logic}, 50:510--530, 1985.

\bibitem[\protect\citeauthoryear{Baader \bgroup \em et al.\egroup
  }{1999}]{baader1999computing}
F.~Baader, R.~K{\"u}sters, and R.~Molitor.
\newblock Computing least common subsumers in description logics with
  existential restrictions.
\newblock In {\em IJCAI'99}, pages 96--101. Morgan-Kaufmann, 1999.

\bibitem[\protect\citeauthoryear{Baader}{2003}]{Baad03nocrossref}
F.~Baader.
\newblock {D}escription {L}ogic terminology.
\newblock In F.~Baader, D.~Calvanese, D.~McGuinness, D.~Nardi, and P.~F.
  Patel-Schneider, editors, {\em The {D}escription {L}ogic Handbook: Theory,
  Implementation, and Applications}, pages 485--495. Cambridge University
  Press, 2003.

\bibitem[\protect\citeauthoryear{Bloch and Lang}{2002}]{TCIS-02}
I.~Bloch and J.~Lang.
\newblock Towards {M}athematical {M}orpho-{L}ogics.
\newblock In B.~Bouchon-Meunier, J.~Gutierrez-Rios, L.~Magdalena, and R.~Yager,
  editors, {\em Technologies for Constructing Intelligent Systems}, pages
  367--380. Springer, 2002.

\bibitem[\protect\citeauthoryear{Bloch \bgroup \em et al.\egroup
  }{2004}]{IB:KR-04}
I.~Bloch, R.~Pino-P\'erez, and C.~Uzcategui.
\newblock A {U}nified {T}reatment of {K}nowledge {D}ynamics.
\newblock In {\em International Conference on the Principles of Knowledge
  Representation and Reasoning, KR2004}, pages 329--337, Canada, 2004.

\bibitem[\protect\citeauthoryear{Dalal}{1988}]{DALA-88}
M.~Dalal.
\newblock Investigations into a {T}heory of {K}nowledge {B}ase {R}evision:
  {P}reliminary {R}eport.
\newblock In {\em AAAI'88}, pages 475--479, 1988.

\bibitem[\protect\citeauthoryear{Delgrande and Peppas}{2015}]{delgrande2015}
J.~P. Delgrande and P.~Peppas.
\newblock Belief revision in {Horn} theories.
\newblock {\em Artificial Intelligence}, 218:1--22, 2015.

\bibitem[\protect\citeauthoryear{{Distel} \bgroup \em et al.\egroup
  }{2014a}]{FD:TechRep-14}
F.~{Distel}, J.~{Atif}, and I.~{Bloch}.
\newblock Concept dissimilarity based on tree edit distance and morphological
  dilation.
\newblock Technical Report 2014D001, Telecom ParisTech - CNRS LTCI, February
  2014.

\bibitem[\protect\citeauthoryear{{Distel} \bgroup \em et al.\egroup
  }{2014b}]{FD:ECAI-14}
F.~{Distel}, J.~{Atif}, and I.~{Bloch}.
\newblock Concept dissimilarity based on tree edit distance and morphological
  dilations.
\newblock In {\em European Conference on Artificial Intelligence (ECAI)}, pages
  249--254, Prag, Czech Republic, 2014.

\bibitem[\protect\citeauthoryear{{Distel} \bgroup \em et al.\egroup
  }{2014c}]{FD:KR-14}
F.~{Distel}, J.~{Atif}, and I.~{Bloch}.
\newblock Concept dissimilarity with triangle inequality.
\newblock In {\em 14th International Conference on Principles of Knowledge
  Representation and Reasoning}, pages 614--617, Wien, Austria, July 2014.

\bibitem[\protect\citeauthoryear{Flouris \bgroup \em et al.\egroup
  }{2005}]{flouris2005applying}
G.~Flouris, D.~Plexousakis, and G.~Antoniou.
\newblock On applying the {AGM} theory to {DLs} and {OWL}.
\newblock In {\em The Semantic Web--ISWC 2005}, pages 216--231. Springer, 2005.

\bibitem[\protect\citeauthoryear{Flouris \bgroup \em et al.\egroup
  }{2006}]{flouris2006inconsistencies}
G.~Flouris, Z.~Huang, J.~Pan, D.~Plexousakis, and H.~Wache.
\newblock Inconsistencies, negations and changes in ontologies.
\newblock In {\em 21st AAAI National Conference on Artificial Intelligence},
  pages 1295--1300, 2006.

\bibitem[\protect\citeauthoryear{G{\"a}rdenfors}{2003}]{gardenfors2003belief}
P.~G{\"a}rdenfors.
\newblock {\em Belief revision}, volume~29.
\newblock Cambridge University Press, 2003.

\bibitem[\protect\citeauthoryear{Gorogiannis and
  Hunter}{2008}]{Gorogiannis2008a}
N.~Gorogiannis and A.~Hunter.
\newblock {Merging First-Order Knowledge using Dilation Operators}.
\newblock In {\em Fifth International Symposium on Foundations of Information
  and Knowledge Systems, FoIKS'08}, volume LNCS 4932, pages 132--150, January
  2008.

\bibitem[\protect\citeauthoryear{Grove}{1988}]{grove1988two}
A.~Grove.
\newblock Two modellings for theory change.
\newblock {\em Journal of philosophical logic}, 17(2):157--170, 1988.

\bibitem[\protect\citeauthoryear{Katsuno and
  Mendelzon}{1991}]{KatsunoMendelzon91}
H.~Katsuno and A.~O. Mendelzon.
\newblock Propositional {K}nowledge {B}ase {R}evision and {M}inimal {C}hange.
\newblock {\em Artificial Intelligence}, 52:263--294, 1991.

\bibitem[\protect\citeauthoryear{Meyer \bgroup \em et al.\egroup
  }{2005}]{meyer2005knowledge}
T.~Meyer, K.~Lee, and R.~Booth.
\newblock Knowledge integration for description logics.
\newblock In {\em AAAI}, volume~5, pages 645--650, 2005.

\bibitem[\protect\citeauthoryear{Qi \bgroup \em et al.\egroup
  }{2006a}]{qi2006revision}
G.~Qi, W.~Liu, and D.~Bell.
\newblock A revision-based approach to handling inconsistency in description
  logics.
\newblock {\em Artificial Intelligence Review}, 26(1-2):115--128, 2006.

\bibitem[\protect\citeauthoryear{Qi \bgroup \em et al.\egroup }{2006b}]{qi2006}
G.~Qi, W.~Liu, and D.~A. Bell.
\newblock Knowledge base revision in description logics.
\newblock In {\em Logics in Artificial Intelligence}, volume LNAI 4160, pages
  386--398, 2006.

\bibitem[\protect\citeauthoryear{Ribeiro and Wassermann}{2014}]{ribeiro2014}
M.~M. Ribeiro and R.~Wassermann.
\newblock Minimal change in {AGM} revision for non-classical logics.
\newblock In {\em International Conference on Principles of Knowledge
  Representation and Reasoning (KR'14)}, pages 657--660, 2014.

\bibitem[\protect\citeauthoryear{Ribeiro \bgroup \em et al.\egroup
  }{2013}]{ribeiro2013}
M.~M. Ribeiro, R.~Wassermann, G.~Flouris, and G.~Antoniou.
\newblock Minimal change: Relevance and recovery revisited.
\newblock {\em Artificial Intelligence}, 201:59--80, 2013.

\bibitem[\protect\citeauthoryear{Schlobach and Cornet}{2003}]{schlobach2003non}
S.~Schlobach and R.~Cornet.
\newblock Non-standard reasoning services for the debugging of description
  logic terminologies.
\newblock In {\em IJCAI}, volume~3, pages 355--362, 2003.

\bibitem[\protect\citeauthoryear{Schlobach \bgroup \em et al.\egroup
  }{2007}]{schlobach2007debugging}
S.~Schlobach, Z.~Huang, R.~Cornet, and F.~Van Harmelen.
\newblock Debugging incoherent terminologies.
\newblock {\em Journal of Automated Reasoning}, 39(3):317--349, 2007.

\end{thebibliography}

\end{document}